\documentclass[twoside]{article}
\usepackage[accepted]{publicpaper}
\usepackage[utf8]{inputenc}
\usepackage{amsthm, amssymb, amsmath}
\usepackage{mathtools}
\usepackage{graphicx}
\usepackage[colorlinks=true,linkcolor=black, citecolor=blue, breaklinks]{hyperref}
\usepackage[ruled]{algorithm2e}
\usepackage{dsfont}
\usepackage{subcaption}
\usepackage{chngcntr}
\usepackage{apptools}
\usepackage[backend=bibtex, hyperref=true, style=authoryear,maxcitenames=2]{biblatex}

\usepackage[final]{fixme}
\usepackage{tikz}
\usetikzlibrary{positioning}
\bibliography{uml}
\fxsetup{inline,nomargin,theme=color}
\AtAppendix{\counterwithin{lemma}{section}}

\DeclareCiteCommand{\cite}
  {\usebibmacro{prenote}}
  {\usebibmacro{citeindex}%
   \printtext[bibhyperref]{\usebibmacro{cite}}}
  {\multicitedelim}
  {\usebibmacro{postnote}}

\DeclareCiteCommand*{\cite}
  {\usebibmacro{prenote}}
  {\usebibmacro{citeindex}%
   \printtext[bibhyperref]{\usebibmacro{citeyear}}}
  {\multicitedelim}
  {\usebibmacro{postnote}}

\DeclareCiteCommand{\parencite}[\mkbibparens]
  {\usebibmacro{prenote}}
  {\usebibmacro{citeindex}%
    \printtext[bibhyperref]{\usebibmacro{cite}}}
  {\multicitedelim}
  {\usebibmacro{postnote}}

\DeclareCiteCommand*{\parencite}[\mkbibparens]
  {\usebibmacro{prenote}}
  {\usebibmacro{citeindex}%
    \printtext[bibhyperref]{\usebibmacro{citeyear}}}
  {\multicitedelim}
  {\usebibmacro{postnote}}

\DeclareCiteCommand{\footcite}[\mkbibfootnote]
  {\usebibmacro{prenote}}
  {\usebibmacro{citeindex}%
  \printtext[bibhyperref]{ \usebibmacro{cite}}}
  {\multicitedelim}
  {\usebibmacro{postnote}}

\DeclareCiteCommand{\footcitetext}[\mkbibfootnotetext]
  {\usebibmacro{prenote}}
  {\usebibmacro{citeindex}%
   \printtext[bibhyperref]{\usebibmacro{cite}}}
  {\multicitedelim}
  {\usebibmacro{postnote}}
  
  \newtheorem*{theorem*}{Theorem}
  \newtheorem*{itheorem}{Informal Theorem}
  \newtheorem{question}{Question}
  
\begin{document}
\newtheorem{lemma}{Lemma}
\newtheorem{corollary}{Corollary}
\newtheorem{theorem}{Theorem}
\newtheorem{proposition}{Proposition}

\theoremstyle{definition}
\newtheorem{definition}{Definition}

%

%

\twocolumn[

\aistatstitle{Optimality of Approximate Inference Algorithms on Stable Instances}
\aistatsauthor{ Hunter Lang \And David Sontag \And Aravindan Vijayaraghavan}
\aistatsaddress{MIT \And MIT \And Northwestern University} ] 

\begin{abstract}
  Approximate algorithms for structured prediction problems---such as
  LP relaxations and the popular $\alpha$-expansion algorithm
  \parencite{BoykovExpansion}---typically far exceed
  their theoretical performance guarantees on real-world
  instances. These algorithms often find solutions that are very close
  to optimal. The goal of this paper is to partially explain the
  performance of $\alpha$-expansion and an LP relaxation algorithm on
  MAP inference in Ferromagnetic Potts models (FPMs). Our main results
  give \emph{stability} conditions under which these two algorithms
  provably recover the optimal MAP solution. These theoretical results
  complement numerous empirical observations of good performance.
\end{abstract}

\section{INTRODUCTION}
For many problems in machine learning, there is a large gap between
the theoretical guarantees offered by the best algorithms and the
empirical performance of those algorithms on real data. For instance,
many MAP inference problems reduce to well-studied combinatorial
optimization problems that are computationally hard in the worst-case;
in practice, however, heuristic approaches often obtain solutions that
far surpass their worst-case guarantees. While worst-case analysis has
been the method of choice in theoretical computer science to reason
about algorithms, beyond-worst-case paradigms like average-case
analysis, smoothed analysis, and implicit assumptions like
stability have become increasingly popular in recent
years~\parencite{BS92,McSherry,ST00,BLstable,BBGStability}. Reconciling
this large gap between theory and practice is an important challenge
in machine learning.


Many tasks in modern machine learning---especially in computer vision
and natural language processing---are framed and solved as structured
prediction problems \parencite{structurebook, sp2, sp3}, where the
local structure of an instance can be used to inform global
decisions. Stereo vision is one such problem: given two input images
$L$ and $R$ (one left, one right), the task is to output a disparity
value for each pixel in the left image that tells how much that pixel
moved between $L$ and $R$. If two neighboring pixels have similar
intensities, the output should give them similar
disparities. Undirected graphical models, also known as Markov Random
Fields, provide a powerful framework for performing this type of
structured prediction. 

Solving the MAP inference problem in a Markov Random Field (MRF) gives
the maximum-probability configuration of variables (e.g., the set of
pixel disparities with maximum probability) taking into account the
interaction effects between nearby variables. An MRF is represented
using a graph $G = (V,E)$ in which each vertex $u \in V$ represents a
random variable that can take values (labels) in the discrete set
$L=\{1, 2, \dots, k\}$, and edges in $E$ represent direct dependencies
between different random variables. We consider
\emph{pairwise} MRFs, where dependencies are only along edges. If we let $g$
be a labeling that maps $V$ to $L$, we can write the MAP inference
task for a pairwise MRF in \emph{energy minimization} form as follows:
\begin{equation}
\label{eq:MRFEM}
\min_{g} \sum_{u\in V} c(u, g(u)) + \smashoperator{\sum_{(u,v) \in E}} \theta_{(u,v)}(g(u), g(v)),
\end{equation}
Here we can interpret $c(u, i)$ as the ``node cost'' of assigning
label $i$ to vertex $u$ and $\theta_{(u,v)}(i, j)$ as the ``edge
cost'' of simultaneously assigning label $i$ to $u$ and label $j$ to
$v$. In stereo vision, the observed pixel intensities can be used to
estimate the node and edge costs
\parencite{BoykovExpansion}. Computing the MAP assignment
corresponding to \eqref{eq:MRFEM} is known to be NP-hard for many
classes of MRFs \parencite{maphard}.

A well-studied special case of \eqref{eq:MRFEM} that has been
successful in practice is the \emph{Ferromagnetic Potts Model}
(FPM). Here, each edge cost function is a nonnegative \emph{weight}
$w(u,v) \in \mathbb{R}_{\ge 0}$ if $u$ and $v$ are assigned different
labels, and $0$ otherwise. We can assume without loss of generality
that the node costs are also nonnegative. The energy minimization
problem in an FPM is:
\begin{equation}
\label{eq:FPMEM}
\min_{g}Q(g) = \min_g\sum_{u \in V}c(u,g(u)) +
\smashoperator{\sum_{\substack{(u,v) \in E\\g(u) \neq g(v)}}}w(u,v),
\end{equation}
where we define $Q(g)$ to be the objective of labeling $g$. The
problem \eqref{eq:FPMEM} is known in theoretical computer science as
\textsc{Uniform Metric Labeling}, and is known to be NP-hard in the
worst-case \parencite{UMLKT}.  While polynomial-time inference
algorithms exist for MRFs with simple structure---like low tree-width
or submodularity---most graphical models that arise in real-world
applications do not have such simple structure.  In recent years,
though, much work has gone into finding tractable model classes and
efficient approximation algorithms for MAP inference.

Linear programming (LP) relaxations give one such class of algorithms.
These algorithms relax the MAP problem to a linear program, typically
by replacing constraints of the form $x \in \{0,1\}$ with $x \in [0,1]$,
then round the (potentially fractional) relaxed solutions back to
integral ones.
In fact, solutions to these linear programming relaxations often turn
out to be mostly integral for instances that arise in practice on
applications like stereo vision~\parencite{lp1, SontagEtAl_uai08, lp2,
  lp4, lp3}. This stands in stark contrast to our theoretical
understanding of linear programming relaxations on worst-case
instances\footnote{The standard LP relaxation for \textsc{Uniform
    Metric Labeling} has an integrality gap of $2$ in the worst
  case~\parencite{UMLKT,ugchard}.}.

Introduced by \textcite{BoykovExpansion}, the $\alpha$-expansion
algorithm is a simple and popular combinatorial algorithm for
approximate MAP inference. It works by iteratively improving an
initial labeling, each time trying to find the optimal ``expansion''
of a label $\alpha$. It is a \emph{local search} algorithm, and it may
get stuck in local energy minima.  Empirically, however,
$\alpha$-expansion seems to avoid bad local minima. 
\textcite{BoykovExpansion} apply the algorithm to stereo vision---they
construct a Ferromagnetic Potts Model from the images $L$ and $R$ and
use $\alpha$-expansion to find an approximate MAP solution, which
gives a disparity value for each pixel. Surprisingly, the solutions
returned by $\alpha$-expansion are strikingly similar to the MAP
solutions, and the output of the algorithm depends very little on the
initial input labeling, even though the algorithm is a 2-approximation
for \eqref{eq:FPMEM} in the worst case. This good performance has led
to wide adoption of the $\alpha$-expansion algorithm in practice.

The near-integrality of LP relaxations and the outstanding performance
of the $\alpha$-expansion algorithm on real-world data lead to the
following compelling question:
\begin{question}
Why do heuristics for MAP inference perform so much better in practice
than their worst-case theoretical guarantees suggest? Can we identify
properties of real-world instances that make them tractable?
\end{question}

Real world problem instances must have structure that worst-case ones
do not. To reconcile this large gap between theory and practice, we
study a structural property of these instances called \emph{stability}
which we think may be key to understanding their tractability.

For many real-world instances, the ground-truth corresponds to a MAP
assignment (optimal solution) that ``stands out''---the optimal
solution is unique and robust to small changes or errors in the
instance specification.  The edge costs involved in the objective are
often imprecise and may only be rough estimates of the similarity
between endpoints. Hence we are interested in finding the optimal
solution only if the instance is stable to errors or perturbations of
the edge costs.  \textcite{BLstable} introduced a formal definition of
stability in the context of graph partitioning problems to capture
instances with a clear ``ground-truth'' solution that does not change
under small multiplicative perturbations to the input.

\begin{definition}[$(\beta, \gamma)$-perturbation]
  Given a weight function on the edges $w: E \to \mathbb{R}_{\ge 0}$,
  a weight function $w'$ is called a $(\beta, \gamma)$-perturbation of
  $w$ iff for any $(u,v) \in E$, $$\frac{1}{\beta}w(u,v) \le w'(u,v)
  \le \gamma w(u,v).$$
\end{definition}
A \textsc{Uniform Metric Labeling} instance with edge weights $w$ is
said to be {\em $(\beta,\gamma)$-stable} iff the optimal solution
$g^*: V \to [k]$ is unique, and it remains unchanged for any
$(\beta,\gamma)$-perturbation $w'$ of the edge weights---that is,
$g^*$ is the unique optimal solution for the instance with edge costs
$w'$ (see Section~\ref{sec:stability} for a formal definition). Note
that such an instance only needs to be stable to multiplicative
perturbations of the edge weights and not to perturbations of the node
costs.

It may sometimes be too strong to assume that the optimal solution to
the perturbed instance remains completely unchanged. In practice, when
the edge costs are perturbed, the optimal solution may change a bit;
however, there is often a {\em stable region} of each instance where
the MAP assignment remains optimal. We capture such instances by
introducing a weaker assumption called $(\beta,\gamma,S)$-weak
stability: roughly speaking, weakly stable instances have a region $S$
on which good solutions always agree with the MAP assignment, even
under edge weight perturbations (see Definition~\ref{def:weakstab} for
a formal definition). We now present our results for stable and weakly
stable instances.

\paragraph{Our Results.} 
We give provable guarantees for a natural LP relaxation and the
$\alpha$-expansion algorithm when the input \textsc{Uniform Metric
  Labeling} instance is sufficiently stable.
Our first result shows that a standard LP relaxation for
\textsc{Uniform Metric Labeling} is exact on sufficiently stable
instances (please see Theorem~\ref{thm:strongthm} for a formal
statement).
\begin{itheorem}
\label{ithm: LP}
Solving the LP relaxation \eqref{eq:umllp} gives the MAP solution for
any $(2,1)$-stable instance of \textsc{Uniform Metric Labeling}.
\end{itheorem}

Our next result gives provable guarantees in the much more general
setting of weak stability: $\alpha$-expansion recovers the optimal
solution on the ``stable region'' of \emph{any} \textsc{Uniform Metric
  Labeling} instance (please see Theorem~\ref{thm:weakstab} for a
formal statement).

\begin{itheorem}
\label{ithm: alphaexpansion}
Given any \textsc{Uniform Metric Labeling} instance that is
$(1,2,S)$-weakly stable, the $\alpha$-expansion algorithm recovers the
optimal solution on $S$. If the entire instance is $(1,2)$-stable
(i.e. $S = V$), $\alpha$-expansion finds the optimal solution.
\end{itheorem}

The above theorem shows that recovering the MAP assignment in the
stable region $S$ is tractable in polynomial time. While this implies
recovery of the entire MAP assignment under the standard notion of
stability (i.e., $S=V$), the theorem gives recovery guarantees even
when the whole instance is not stable. This gives a partial guarantee
for $\alpha$-expansion, an iterative hill-climbing algorithm.

While both the results require stability to multiplicative
perturbations up to a factor $2$, these two stability conditions
($(2,1)$-stability and $(1,2)$-stability) seem qualitatively
different. We show in Section \ref{sec:counter} that the optimality of
both of these algorithms breaks down when the stability conditions of
the two theorems are switched.

\section{RELATED WORK}
Instance stability has been studied in the context of graph
partitioning problems like Max-Cut~\parencite{BLstable,
  BDLS,MMWC4Stab} and minimum multiway
cut~\parencite{MMWC4Stab,TwoStab}, clustering problems like $k$-means
and $k$-median~\parencite{ABSStability,
  BalcanLiangStability,WhiteStability,TwoStab,additiveStability}, and the traveling
salesman problem~\parencite{tspStability}.

Our work is inspired by \textcite{MMWC4Stab} and
\textcite{TwoStab}. \textcite{MMWC4Stab} developed a general framework
to analyze stable instances of graph partitioning problems, showing
that if there exists a convex relaxation and a rounding scheme for a
problem satisfying certain properties, then the convex relaxation is
exact for sufficiently stable instances of the problem. They also
designed a new polynomial time iterative algorithm for ``weakly
stable'' instances of the problem, where the optimal solution can
change slightly under perturbations of the weights. The amount of
stability required depends on the guarantees of the rounding
scheme. \textcite{MMWC4Stab} applied this framework to the
\textsc{Minimum Multiway Cut} problem~\parencite{MMWC}, which is a
special case of \textsc{Uniform Metric Labeling}. They give a
polynomial-time algorithm for 4-stable\footnote{In \textsc{Minimum
    Multiway Cut}, the stability parameter is given as the product
  $\beta\gamma$, since a $(\beta, \gamma)$-perturbation is equivalent
  to a $(\gamma, \beta)$-perturbation in that setting.}  instances of
\textsc{Minimum Multiway Cut}. \textcite{TwoStab} also studied
\textsc{Minimum Multiway Cut} and designed a better rounding scheme to
give provable guarantees for $2 - 2/k$-stable instances.

We use the same framework as \textcite{MMWC4Stab, TwoStab} to prove
integrality of the LP relaxation \eqref{eq:umllp}
(Theorem~\ref{thm:strongthm}). However, there are several new
technical challenges that we need to address to prove our results, and
we briefly describe them below.

Unlike the \textsc{Minimum Multiway Cut} problem, there are two
different costs in \textsc{Uniform Metric Labeling}: edge weights and
node costs. Our notion of stability only assumes that the optimal
solution does not change under perturbations to edge weights; we make
no assumptions about perturbations to node costs.

While there is a simple reduction from an instance of \textsc{Uniform
  Metric Labeling} to an instance of \textsc{Minimum Multiway
  Cut}~\parencite{BoykovGPM}, it converts all the node costs into edge
weights. Using this reduction would effectively force us to assume
stability with respect to perturbations of node costs as
well. Further, the reduction creates edges of very large weight, so
the stability condition required becomes very stringent. To address
these challenges, we first show the existence of a new rounding scheme
for the standard LP relaxation that delicately trades off the loss to
the LP solution on the node costs with the loss on the edge
weights. By contrast, the rounding scheme of \parencite{TwoStab} may
incur too much loss on the node costs relative to the LP, but node
costs are irrelevant for \textsc{Minimum Multiway Cut}. These new
rounding guarantees suffice for our theorems.

To obtain provable optimality guarantees for the $\alpha$-expansion
algorithm, we relate the improvement in each step of
$\alpha$-expansion to the performance of a linear program on a
modified instance. We carefully perturb the weights of the given
instance (based on the current solution), to argue that if the current
solution is {\em not} optimal on the stable portion of the instance,
then an $\alpha$-expansion step finds a new solution with smaller
cost. This proves that the $\alpha$-expansion algorithm recovers the
optimal solution restricted to the stable region.

\section{BACKGROUND}
\label{sec:prelim}
\subsection{$\alpha$-expansion}
Solving \eqref{eq:FPMEM} is NP-hard \parencite{BoykovExpansion,
  UMLKT}, but several efficient approximation algorithms exist. In
this work, we focus on a simple combinatorial algorithm known as
$\alpha$-expansion, which works by iteratively improving an initial
labeling with ``expansion moves.''
\begin{definition}[Expansion Move]
  Let $f$ be an arbitrary labeling $f: V \to L$; $f$ gives rise to a
  partition $S^f_1, \ldots, S^f_k$ of the vertices, where $v \in S^f_i$
  if and only if $f(v) = i$. We call a labeling $g: V \to L$ an
  \emph{$\alpha$-expansion} of $f$ if the following two conditions
  hold:
  \begin{align*}
  S^f_{\alpha} &\subset S^g_{\alpha}; \;\;\;\;S^g_{i} \subset S^f_{i},\ \ i \ne \alpha.
  \end{align*}
In other words, the set of vertices labeled $\alpha$ may grow from $f$
to $g$, and all other label sets $S_{i}$ may not grow.
\end{definition}
\begin{algorithm}[t]
  Initialize a labeling $f: V \to L$.

  Set \texttt{continue = True}.

  \While{\texttt{continue}} {
  Set \texttt{continue = False}.

  \For{$\alpha \in L$} {
    Find the optimal $\alpha$-expansion $h$ of $f$.

    \If{$Q(h) < Q(f)$} {
      Set $f = h$ and \texttt{continue = True}.
    }
  }
}
  \Return $f$.
  \caption{$\alpha$-expansion Algorithm}
  \label{alg:expansion}
\end{algorithm}
The full procedure is described in Algorithm
\ref{alg:expansion}. \textcite{BoykovExpansion} show that the optimal
$\alpha$-expansion move for a given labeling $f$ and label $\alpha$
can be found by solving a minimum cut problem in an auxiliary graph
$G_{\alpha,f}$. 
Finally, they show that $\alpha$-expansion is a 2-approximation for
\eqref{eq:FPMEM}.

Due to its simplicity, good empirical performance, and the
availability of very fast implementations, Algorithm
\ref{alg:expansion} has seen widespread use in practice
\parencite{exp1, exp4}.

\subsection{LP Relaxations}
Linear programming (LP) relaxations are also commonly used to find
approximate MAP solutions. We will make extensive use of the following
LP relaxation of \eqref{eq:FPMEM} as a tool to analyze
$\alpha$-expansion and as an algorithm itself:
\begin{equation}
  \label{eq:umllp}
  \begin{split}
  \min_{\{\bar{u}\}} \sum_{u\in V}\sum_{i \in L}c(u,i)\bar{u}_i + \sum_{(u,v)\in E}w(u,v)d(u,v)\\
    \begin{aligned}
      \text{s.t.}\;
      & \sum_i\bar{u}_i = 1, \; 
      &&\forall u\in V,\ \forall i \in L\\
      & d(u,v) = \frac{1}{2}||\bar{u} - \bar{v}||_1,\; 
      &&\forall (u,v) \in E\\
      & \bar{u}_i \ge 0,\; 
      &&\forall u \in V,\ i \in L.
    \end{aligned}
  \end{split}
\end{equation}
Here $\bar{u}$ is the length-$k$ vector of fractional assignments at
node $u$. Note that the second constraint can easily be linearized
using edge variables. With a slight abuse of notation, we
say $$Q(\{\bar{u}\}) = \sum_{u\in V}\sum_{i \in L}c(u,i)\bar{u}_i +
\sum_{(u,v)\in E}w(u,v)d(u,v).$$ Any integer labeling $f$ is also a
feasible point of \eqref{eq:umllp}; $f$ corresponds to
$\{\bar{u}^f\}$, where $\bar{u}^f_i = 1$ if $f(u) = i$ and 0
otherwise. In that case, the distance $d^f(u,v) = 1$ if $f(u) \ne
f(v)$ and 0 otherwise, and $Q(\{\bar{u}^f\}) = Q(f)$. We say the LP
relaxation is \emph{tight} if an optimal solution to \eqref{eq:umllp}
is an integer solution (i.e. $\bar{u}_i \in \{0,1\}$ for all $u$ and
$i$).
On Ferromagnetic Potts Models, the relaxation \eqref{eq:umllp} is
equivalent to the \emph{local polytope} relaxation commonly studied in
MAP inference \parencite{wainwrightjordan, almostbalanced}. The
appendix contains a proof of that equivalence.

\subsection{Stability}\label{sec:stability}

We now formally define stable instances of \textsc{Uniform Metric
  Labeling}.
\begin{definition}[$(\beta, \gamma)$-stable]
  An instance of \textsc{Uniform Metric Labeling} $(G, c, w, L)$ with
  graph $G$, node costs $c$, weights $w$, labels $L$, and optimal
  integer solution $g$ is called $(\beta, \gamma)$-stable if for any
  $(\beta, \gamma)$-perturbation $w'$ of $w$, and any labeling $h \ne
  g$, $Q'(h) > Q'(g),$ where $Q'$ is the objective with costs $c$ and
  weights $w'$.

That is, $g$ is the unique optimal solution in any $(\beta,
\gamma)$-perturbation. Note that node costs $c(u,i)$ are not perturbed
in this definition. Requiring stability under perturbations of the
costs $c(u,i)$ would lead to a stronger condition on the input
instance; perturbations to $w$ are sufficient for our theorems. As
$\beta$ and $\gamma$ increase, the stability condition becomes
increasingly strict.
\end{definition}

The next definition captures a broader, more local version of
stability, where the instance is stable with respect to a region $S
\subset V$.
\begin{definition}[$(\beta, \gamma, S)$-weakly-stable]
  \label{def:weakstab}
  For some set $S \subset V$, an instance $(G, c, w, L)$ of
  \textsc{Uniform Metric Labeling} with optimal solution $g$ is said
  to be $(\beta, \gamma, S)$-weakly-stable if for any $(\beta,
  \gamma)$-perturbation $w'$ of the weights $w$ and any labeling $h:
  V\to L$, \begin{equation}
    \label{eq:weakstab}
    h_S \ne g_S \implies Q'(h) > Q'(g).
  \end{equation}
  Here $h_S$ and $g_S$ are the restrictions of $h$ and $g$ to $S$, and
  $Q'$ is the objective with costs $c$ and weights $w'$. We call $S$
  the \emph{stable set} or \emph{stable region} of the instance. The
  weak stability property says that for any $(\beta,
  \gamma)$-perturbation of the edge weights, any solution that
  disagrees with the optimal solution $g$ on the stable set must have
  a worse objective value. Note that a $(\beta, \gamma,
  V)$-weakly-stable instance is $(\beta, \gamma)$-stable.
\end{definition}

\section{LP RELAXATION AND (2,1)-STABLE INSTANCES}
\label{sec:strongstab}
In this section we prove that the LP relaxation \eqref{eq:umllp},
which as we mentioned is equivalent to the local consistency relaxation, is
tight on (2,1)-stable instances. Our proof follows the framework
introduced by \textcite{MMWC4Stab} and \textcite{TwoStab}. We assume
\emph{for a contradiction} that the LP is fractional at some node, then use
the probabilistic method to show that there must be a labeling that
violates stability of the instance. To construct this violating
labeling, we build a randomized rounding algorithm for
\eqref{eq:umllp} that provides certain probabilistic guarantees. We
show that on a carefully constructed fractional input, this rounding
algorithm outputs a solution that violates stability in
expectation. Thus, there must be some labeling that violates
stability, and therefore the optimal LP solution must take integer
values at every node. The rounding algorithm defined below and the
proof techniques in this section are also used in Section
\ref{sec:weakstab} to analyze $\alpha$-expansion.

To begin, we describe the rounding procedure. This algorithm only
works on inputs that are ``close'' to integer solutions; we will show
how to construct these so-called $\varepsilon$-close inputs shortly.
\begin{definition}[$\varepsilon$-close]
Fix $\varepsilon < \frac{1}{2}$. A solution $\{\bar{u}\}$ to LP
\eqref{eq:umllp} is $\varepsilon$-close to an integer labeling if for
each $u \in V$, there exists some $j$ such that $\bar{u}_j \ge 1 -
\varepsilon$. Because of the simplicial constraint on each $\bar{u}$,
this index $j$ is unique; we can therefore refer to it as $j(u)$.
\end{definition}
The rounding algorithm $\mathcal{R}$ is defined in Algorithm \ref{alg:rounding}.
\begin{algorithm}[t]
  Define $P_i$ as the set of vertices labeled $i$.
  
  Let $\varepsilon = 1/(10k)$ and $\theta = 6/(5k)$. Note $\theta > \varepsilon$.

  Choose $r \in (0,\theta)$ uniformly at random.

  Choose $i \in \{1,\ldots, k\}$ uniformly at random.

  Apply the following rule to every node $u \in V$:

  \ \ \ \ If $\bar{u}_i < r$, add $u$ to $P_{j(u)}$. Otherwise, add $u$ to $P_i$.

  Return the partition $(P_1,\ldots,P_k)$.
  \caption{Rounding Algorithm $\mathcal{R}$}
  \label{alg:rounding}
\end{algorithm}

The following properties of $\mathcal{R}$ will help construct a
stability-violating labeling:
\begin{lemma}[Rounding Guarantees]
  \label{lem:roundingguar}
  Let $h$ be the (random) output of Algorithm \ref{alg:rounding} on an
  $\varepsilon$-close solution $\{\bar{u}\}$. Then:
  \begin{align*}
  \Pr[h(u) \ne j(u)] &\ge \frac{5}{6}(1-\bar{u}_{j(u)})\\
  \Pr[h(u) = i] &\le \frac{5}{6}\bar{u}_i,\;\; \forall i \ne j(u)\\
  \Pr[(u,v) \mbox{ not cut}] &\ge \frac{5}{6}(1-d(u,v))\\
  \Pr[(u,v) \mbox{ cut}] &\le \frac{5}{3}d(u,v),
  \end{align*}
  where $j(u)$ is the index such that $\bar{u}_{j(u)} \ge 1 - \varepsilon$.
\end{lemma}
The appendix contains a proof of these guarantees. Note that since the
rounding only works on $\varepsilon$-close solutions, we cannot turn
these properties into an approximation algorithm. We can now use
$\mathcal{R}$ to prove the main theorem of this section:
\begin{theorem}
  \label{thm:strongthm}
  On a (2,1)-stable instance of \textsc{Uniform Metric Labeling} with
  optimal integer solution $g$, the LP relaxation \eqref{eq:umllp} is
  tight.
\end{theorem}
\begin{proof}
Assume for a contradiction that the optimal LP solution
$\{\bar{u}^{LP}\}$ of \eqref{eq:umllp} is fractional. To construct a
stability-violating labeling, we will run Algorithm \ref{alg:rounding}
on a fractional labeling $\{\bar{u}\}$ constructed from
$\{\bar{u}^{LP}\}$ and the optimal integer solution $g$. We then use
Lemma \ref{lem:roundingguar} to show that in expectation, the output
of $\mathcal{R}(\{\bar{u}\})$ must be better than the optimal integer
solution in a particular $(2,1)$-perturbation, which contradicts
$(2,1)$-stability.

Let $\{\bar{u}^g\}$ be the solution to \eqref{eq:umllp} corresponding
to $g$, and define the following $\varepsilon$-close solution
$\{\bar{u}\}$: for every $u$ and every $i$, set $\bar{u}_i =
(1-\varepsilon)\bar{u}^g_i + \varepsilon\bar{u}^{LP}_i$. Note that
$\{\bar{u}\}$ is fractional and $j(u) = g(u)$ for all $u$.

Recall that $E_g$ is the set of edges cut by the optimal solution
$g$. Define the following $(2,1)$-perturbation $w'$ of the weights
$w$:
\begin{equation*}
  w'(u,v) = \begin{cases}
    w(u,v) &(u,v) \in E_g\\    
    \frac{1}{2}w(u,v) &(u,v) \in E\setminus E_g.
  \end{cases}
\end{equation*}
We refer to the objective with modified weights $w'$ as $Q'$ (that is,
$Q'$ is the objective in the instance with weights $w'$ and costs
$c$).

Now let $h = \mathcal{R}(\{\bar{u}\})$. To compare $g$ and $h$, we
will compute $\mathbb{E}[Q'(g) - Q'(h)]$, where the expectation is
over the randomness of the rounding algorithm. By definition,
\begin{align*}
  \mathbb{E}\left[Q'(g) - Q'(h)\right] &= \mathbb{E}[Q'(g) - Q'(h) | h =
    g]\Pr(h = g) \\
  &+ \mathbb{E}[Q'(g) - Q'(h) | h \ne g]\Pr(h \ne g).
\end{align*}
The first term of the sum above is clearly zero. Further, as
$\{\bar{u}\}$ is fractional, the guarantees in Lemma
\ref{lem:roundingguar} imply that $\Pr(h \ne g) > 0$. By
$(2,1)$-stability of the instance, any labeling $h \ne g$ must satisfy
$Q'(h) > Q'(g)$. So stability and fractionality of the LP imply
$\mathbb{E}[Q'(g) - Q'(h)] < 0$.

If we compute $\mathbb{E}[Q'(g) - Q'(h)]$ and simplify using Lemma
\ref{lem:roundingguar} and the definition of $w'$ (see the appendix
for a full derivation), we obtain:
\begin{equation*}
\begin{split}
  \mathbb{E}[Q'(g) - Q'(h)] \ge \frac{5}{6}\left(\sum_{u \in V}c(u,g(u)) + \smashoperator{\sum_{(u,v) \in E_g}}w(u,v)\right.\\
  - \left.\sum_{u \in V}\sum_{i \in L}c(u,i)\bar{u}_i\right. - \left.\smashoperator{\sum_{(u,v) \in E}}w(u,v)d(u,v)\right)
\end{split}
\end{equation*}
The first two terms are simply $Q(g)$, and the last two are the
objective $Q(\{\bar{u}\})$ of the LP solution $\bar{u}$. Since
$\bar{u} = (1 - \varepsilon)\bar{u}^g + \varepsilon\bar{u}^{LP}$ and
$Q(\{\bar{u}^{LP}\}) \le Q(\{\bar{u}^g\})$, the convexity of the LP
objective implies $Q(\{\bar{u}\}) \le Q(\{\bar{u}^g\}) = Q(g)$. So
$\mathbb{E}[Q'(g) - Q'(h)] \ge 0$. But stability of the instance and
fractionality of the LP solution implied $\mathbb{E}[Q'(g) - Q'(h)] <
0$.
\end{proof}

\section{$\alpha$-EXPANSION AND (1,2)-STABLE INSTANCES}
\label{sec:weakstab}
In this section, we study the broader stability condition given by
Definition \ref{def:weakstab}, where the instance may only be stable
with respect to some region $S$. We prove that for $(1, 2,
S)$-weakly-stable instances, the $\alpha$-expansion algorithm recovers
the optimal solution on the stable set $S$. In other words, given any
input $\alpha$-expansion is guaranteed to recover the optimal solution
on the stable portion $S$ of the instance. When $S = V$, this implies
recovery of the entire optimal solution for $(1,2)$-stable instances.

The proof shows that as long as the current labeling maintained by
$\alpha$-expansion does not agree with the optimal solution on the
stable set $S$, there must be an expansion move that decreases the
objective. \textcite{BoykovExpansion} show that $\alpha$-expansion
cannot terminate while there is an expansion move that decreases the
objective. Hence $\alpha$-expansion cannot terminate
until it agrees with the optimal solution on the stable set.

To show how to construct an expansion move that decreases the
objective, we will again use the probabilistic method. We actually
construct the expansion move by using the LP relaxation and rounding
algorithm from the previous section. Indeed, we will use the LP
solution on a modified instance to construct an input $\{\bar{u}\}$ to
the rounding algorithm $\mathcal{R}$. We show that as long as the
current labeling differs from the optimal one on the stable set, there
is some labeling in the support of $\mathcal{R}(\{\bar{u}\})$ that
decreases the objective. The following lemma shows that every labeling
in the support of $\mathcal{R}$ is an expansion move.
\begin{lemma}
    \label{lem:roundexpand}
    Let $\{\bar{u}\}$ be an input to Algorithm \ref{alg:rounding} and
    let $F: V\to L$ be the integer solution to which $\{\bar{u}\}$ is
    $\varepsilon$-close. Then for all labelings $h$ in the support of
    Algorithm \ref{alg:rounding}, there exists an $i$ such that $h$
    is an $i$-expansion of the labeling $F$.
  \end{lemma}
  \begin{proof}
    Algorithm \ref{alg:rounding} makes a random choice of label
    $i$. Then for every vertex $u$, it assigns either label $i$ or
    label $F(u)$ to that vertex. Clearly the set of vertices labeled
    $i$ by $F$ does not decrease, and the only new label assigned is
    $i$. So the output labeling is an $i$-expansion of $F$.
  \end{proof}
Therefore, there must be an expansion move that decreases the
objective as long as the current labeling differs from the optimal one
on the stable set. We can now state the theorem.
\begin{theorem}
\label{thm:weakstab}
  On a $(1,2,S)$-weakly-stable instance $(G, c, w, L)$ with optimal
  solution $g$, let $f$ be the solution output by Algorithm
  \ref{alg:expansion}. Then $f_S = g_S$. That is, $\alpha$-expansion
  recovers the optimal solution on the stable set.\footnote{The proof
    below shows that $Q'(f) \le Q'(g)$ (in the sense of Definition
    \ref{def:weakstab}) is a necessary condition for local optimality
    of $f$ with respect to expansion moves. Interpreted in the
    framework of perturbation-stability, the local optimality argument
    of \textcite{BoykovExpansion} can also be shown to give Theorem
    \ref{thm:weakstab}.} When $S = V$, $\alpha$-expansion recovers the
  full optimal solution.
\end{theorem}
We remark that Algorithm 1 is known to run in polynomial time in $|V|$
and $|L|$ as long as the costs are polynomially bounded.
\textcite{VekslerThesis} shows that it converges in a polynomial
number of iterations when the costs and weights are constant in $|V|$
and $|L|$, (or are integers that are polynomially bounded in $|V|$ and
$|L|$), and each iteration performs $|L|$ maximum flow
computations. In practice, Algorithm 1 typically takes only 2-5
iterations to converge \parencite{BoykovExpansion}.

\begin{proof}
We prove that as long as the labeling at the current iteration,
denoted by $f$, satisfies $f_S \ne g_S$, there exists an expansion
move of $f$ that decreases the objective.  We use Algorithm
\ref{alg:rounding} as a tool to show that this expansion must exist by
constructing a particular input $\{\bar{u}\}$ such that
$\mathcal{R}(\{\bar{u}\})$ has better objective than $f$ in
expectation as long as $f_S \ne g_S$.

Let $\{\bar{u}^{f}\}$ be the solution to LP \eqref{eq:umllp}
corresponding to $f$, and recall that $E_f$ is the set of edges cut by
$f$. Define the following $(1,2)$-perturbation of the weights $w$:
  \begin{equation*}
    w'(u,v) = \begin{cases}
      w(u,v) &(u,v) \in E_f\\
      2w(u,v) &(u,v) \notin E_f
    \end{cases}
  \end{equation*}
Now let $\{\bar{u}^{LP}\}$ be the optimal LP solution to the instance
$(G, c, w', L)$ with these \emph{modified} weights. We construct an
$\varepsilon$-close input $\{\bar{u}\}$ to the rounding algorithm:
$\bar{u} = (1-\varepsilon)\bar{u}^f + \varepsilon\bar{u}^{LP}$. For
each of these labelings, the distance $d$ in the LP relaxation
\eqref{eq:umllp} is given by:
\begin{align*} d^f(u,v) &= \frac{1}{2}||\bar{u}^f -
  \bar{v}^f||_1 = \mathds{1}[f(u) \ne f(v)]\\ d^{LP}(u,v) &=
  \frac{1}{2}||\bar{u}^{LP} - \bar{v}^{LP}||_1\\ d(u,v) &=
  \frac{1}{2}||\bar{u} - \bar{v}||_1.
\end{align*} 
From the definition of $\{\bar{u}\}$ and the triangle inequality,
\begin{equation}
  \label{eq:distineq}
  d(u,v) \le (1-\varepsilon)d^{f}(u,v) + \varepsilon d^{LP}(u,v).
\end{equation}

Let $h = \mathcal{R}(\{\bar{u}\})$ be the random labeling output by
the rounding algorithm $\mathcal{R}$ (Algorithm \ref{alg:rounding}) on
input $\{\bar{u}\}$. We now show that $\mathbb{E}[Q(f) - Q(h)] >
0$. That is, in expectation the rounding algorithm produces a solution
better than $f$ in the original instance.
  \begin{align*}
    \mathbb{E}[Q(f) - Q(h)] &= \sum_{u \in V}c(u, f(u))\Pr[h(u) \ne f(u)] \\
    &+\smashoperator{\sum_{(u,v) \in E_f}}w(u,v)\Pr[(u,v) \mbox{ not cut}] \\
    &-\sum_{u \in V}\sum_{i \ne f(u)}c(u, i)\Pr[h(u) = i] \\
    &-\smashoperator{\sum_{(u,v) \in E \setminus E_f}}w(u,v)\Pr[(u,v) \mbox{ cut}].
  \end{align*}
  Applying the rounding guarantees from Lemma \ref{lem:roundingguar}
  and using the definition of $w'$ (here we need a
  $(1,2)$-perturbation, not a $(2,1)$-perturbation), we obtain the
  following lower bound for $\mathbb{E}[Q(f) - Q(h)]$:
  \begin{align*}
    \mathbb{E}[Q(f) &- Q(h)] \ge \frac{5}{6}\left(\sum_{u \in V}c(u, f(u)) + \smashoperator{\sum_{(u,v) \in E_f}}w'(u,v)\right. \\
    &- \left.\sum_{u \in V}\sum_{i \in L}c(u, i)\bar{u}_i - \smashoperator{\sum_{(u,v) \in E}}w'(u,v)d(u,v)\right)
  \end{align*}
  Writing $f$ as $\{\bar{u}^f\}$,
  \begin{equation*}
    \begin{split}
      \mathbb{E}[Q(f) - Q(h)] &\ge \frac{5}{6} \left(\sum_{u \in V}\sum_{i\in L}c(u, i)\left(\bar{u}^{f}_i - \bar{u}_i\right)\right.\\
      &\left.+ \sum_{(u,v) \in E}w'(u,v)\left(d^{f}(u,v) - d(u,v)\right)\right)
    \end{split}
  \end{equation*}
  By \eqref{eq:distineq}, $d^{f}(u,v) - d(u,v) \ge
  \varepsilon\left(d^{f}(u,v) - d^{LP}(u,v)\right).$ Additionally,
  $\bar{u}^{f}_i - \bar{u}_i = \varepsilon\left(\bar{u}^{f}_i -
  \bar{u}^{LP}_i\right)$ for all $u,\ i$. Then
  \begin{equation*}
    \begin{split}
      \mathbb{E}[Q(f) &- Q(h)] \ge 
      \frac{5}{6}\varepsilon\left(\sum_{u \in V}\sum_{i\in L}c(u, i)\left(\bar{u}^{f}_i - \bar{u}^{LP}_i\right)\right.\\
      &+ \left.\sum_{(u,v) \in E}w'(u,v)\left(d^{f}(u,v) - d^{LP}(u,v)\right)\right)
    \end{split}
  \end{equation*}
  Using the definition of $Q'$ (the objective in the instance $(G, c,
  w', L)$),
  \begin{equation*}
    \mathbb{E}[Q(f) - Q(h)] \ge \frac{5}{6}\varepsilon\left(Q'(f) - Q'(\{\bar{u}^{LP}\})\right).
  \end{equation*}
  Recall that $g$ is the optimal solution in the original
  instance. Since $\{\bar{u}^{LP}\}$ is the optimal LP solution for
  the instance with weights $w'$, $Q'(\{\bar{u}^{LP}\}) \le
  Q'(g)$. Combining, we obtain:
  \begin{equation*}
    \mathbb{E}[Q(f) - Q(h)] \ge \frac{5}{6}\varepsilon\left(Q'(f) - Q'(g)\right)
  \end{equation*}
  
  If $f_S \ne g_S$, by the weak stability of the instance, $Q'(f) >
  Q'(g)$, so in that case $\mathbb{E}[Q(f) - Q(h)] > 0$ and there must
  be some labeling in the support of the rounding algorithm whose
  objective is less than $f$'s. The input to the rounding algorithm
  was $\varepsilon$-close to $f$, so by Lemma \ref{lem:roundexpand},
  every labeling in the support is an expansion move of $f$. So as
  long as $f_S \ne g_S$, some expansion move of $f$ decreases the
  objective. \textcite{BoykovExpansion} show that $\alpha$-expansion
  only terminates when no expansion move decreases the objective,
  hence $f_S = g_S$ when the algorithm terminates.
\end{proof}

\section{COUNTEREXAMPLES}
\label{sec:counter}
The algorithms analyzed in Sections \ref{sec:strongstab} and
\ref{sec:weakstab} give guarantees in two different stability
settings: $(2,1)$-stability, for the LP, and $(1,2)$-stability, for
$\alpha$-expansion. Here we show that each algorithm does not provably
recover the optimal solution in the other stability setting. That is,
the LP relaxation is not tight on $(1,2)$-stable instances, and
$\alpha$-expansion does not always find the optimal solution on
$(2,1)$-stable instances. Stability was tested by checking all
possible labelings in the adversarial perturbation
\begin{equation*}
  w'(u,v) = \begin{cases}
    \frac{1}{\beta}w(u,v) & (u,v) \in E \setminus E_g\\
    \gamma w(u,v) & (u,v) \in E_g.
    \end{cases}
\end{equation*}
A proof that this is sufficient can be found in the appendix, together
with the full details on how the counterexamples were generated.

\paragraph{LP and $(1,2)$-stability.}
\label{sec:lpcounter}

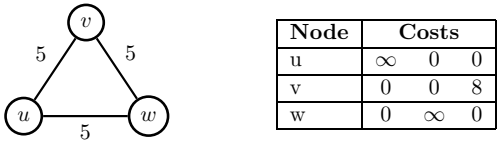
\begin{figure}[ht]
  \centering

\begin{subfigure}{.5\linewidth}
  \centering
  \scalebox{0.83}{  
  \tikzstyle{vertex}=[circle, draw=black, very thick, minimum size=5mm]
  \tikzstyle{edge} = [draw=black, line width=1]
  \tikzstyle{weight} = [font=\normalsize]
  \begin{tikzpicture}[scale=2,auto,swap]
    \foreach \pos /\name in {{(0,0)}/u,{(1,0)}/w,{(0.5,0.75)}/v}
    \node[vertex](\name) at \pos{$\name$};
    \foreach \source /\dest /\weight in {u/w/5}
    \path[edge] (\source) -- node[weight] {$\weight$} (\dest);
    \foreach \source /\dest /\weight/\pos in {u/v/5/{above left}, v/w/5/{above right}}
    \path[edge] (\source) -- node[weight, \pos] {$\weight$} (\dest);
  \end{tikzpicture}
  }
\end{subfigure}%
\begin{subfigure}{.5\linewidth}
  \centering
  \scalebox{0.83}{
\begin{tabular}{|l|ccc|}
\hline
\multicolumn{1}{|c|}{\textbf{Node}} & \multicolumn{3}{|c|}{\textbf{Costs}} \\
\hline
u & $\infty$ & 0        & 0      \\
\hline
v & 0        & 0        & 8      \\
\hline
w & 0        & $\infty$ & 0      \\
\hline
\end{tabular}
}
\end{subfigure}
  \caption{$(1,2)$-stable instance where the LP solution is fractional.}
\label{fig:counter1}
\end{figure}

\begin{table}[ht]
  \centering
  \begin{subtable}{.5\linewidth}
    \centering
    \subcaption{Integer OPT, Obj = 8.0}
    \scalebox{0.83}{    
      \begin{tabular}{|l|ccc|}
        \hline
        \multicolumn{1}{|c|}{\textbf{Node}} & \multicolumn{3}{|c|}{\textbf{Assignment}} \\
        \hline
        u & 0        & 0        & 1      \\
        \hline
        v & 0        & 0        & 1      \\
        \hline
        w & 0        & 0        & 1      \\
        \hline
      \end{tabular}
    }
  \end{subtable}%
  \begin{subtable}{.5\linewidth}
    \centering
    \subcaption{LP OPT, Obj = 7.5}
    \scalebox{0.83}{
      \begin{tabular}{|l|ccc|}
        \hline
        \multicolumn{1}{|c|}{\textbf{Node}} & \multicolumn{3}{|c|}{\textbf{Assignment}} \\
        \hline
        u & 0        & 0.5      & 0.5      \\
        \hline
        v & 0.5        & 0.5        & 0      \\
        \hline
        w & 0.5        & 0 & 0.5      \\
        \hline
      \end{tabular}
    }
  \end{subtable}%
  \caption{Solutions to instance in Figure
    \ref{fig:counter1}. Entries along each row are the assignments
    of the corresponding node to labels 1, 2, and 3, respectively.}\label{tab:solutions1}
\end{table}

Figure \ref{fig:counter1} shows a $(1,2)$-stable instance with three
nodes and three labels. The optimal integer solution and optimal
fractional solution are shown in Table \ref{tab:solutions1}. The
fractional solution has strictly lower objective value. Since no edges
are cut ($E_g = \emptyset$), the adversarial perturbation does not
change any edge weights. The optimal solution is unique, so this
instance is $(1,2)$-stable. Note that $\alpha$-expansion is exact on
this instance---no matter the starting labeling, expanding label 3
gives the optimal solution.

\paragraph{$\alpha$-expansion and $(2,1)$-stability.}
\label{sec:alphacounter}
\begin{figure}[ht]
\centering
\begin{subfigure}{.5\linewidth}
  \centering
  \scalebox{0.83}{  
  \tikzstyle{vertex}=[circle, draw=black, very thick, minimum size=5mm]
  \tikzstyle{edge} = [draw=black, line width=1]
  \tikzstyle{weight} = [font=\normalsize]
  \begin{tikzpicture}[scale=2,auto,swap]
    \foreach \pos /\name in {{(0,0)}/u,{(0,0.75)}/v,{(0.75,0.75)}/w,{(0.75,0)}/x}
    \node[vertex](\name) at \pos{$\name$};
    \foreach \source /\dest /\weight/\pos in {u/v/4/{left}, v/w/3/{above}, w/x/3/{right}}
    \path[edge] (\source) -- node[weight, \pos] {$\weight$} (\dest);
  \end{tikzpicture}
  }
\end{subfigure}%
\begin{subfigure}{.5\linewidth}
  \centering
  \scalebox{0.83}{
    \begin{tabular}{|l|ccc|}
      \hline
      \multicolumn{1}{|c|}{\textbf{Node}} & \multicolumn{3}{|c|}{\textbf{Costs}} \\
      \hline
      u & 2  & 0  & 0  \\
      \hline
      v & 0  & 0  & 100 \\
      \hline
      w & 0  & 100  & 0  \\
      \hline
      x & 100  & 0  & 0  \\
      \hline
    \end{tabular}
  }
\end{subfigure}
\caption{$(2,1)$-stable instance where $\alpha$-expansion does not
  find the optimal solution. }
\label{fig:counter2}
\end{figure}
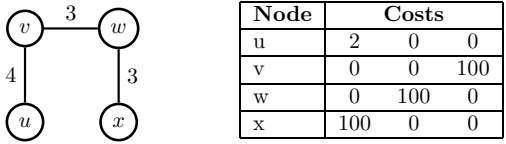

\begin{table}[ht]
  \centering
  \begin{subtable}{.5\linewidth}\centering
    \subcaption{Integer OPT, Obj = 3.0}
    \scalebox{0.83}{
    \begin{tabular}{|l|c|}
      \hline
      \multicolumn{1}{|c|}{\textbf{Node}} & \multicolumn{1}{c|}{\textbf{Label}} \\
      \hline
      u & 2\\
      \hline
      v & 2\\
      \hline
      w & 3\\
      \hline
      x & 3\\
      \hline
    \end{tabular}
    }
  \end{subtable}%
  \begin{subtable}{.5\linewidth}\centering
    \subcaption{Expansion Sol, Obj = 5.0}
    \scalebox{0.83}{
    \begin{tabular}{|l|c|}
      \hline
      \multicolumn{1}{|c|}{\textbf{Node}} & \multicolumn{1}{c|}{\textbf{Label}} \\
      \hline
      u & 1\\
      \hline
      v & 1\\
      \hline
      w & 1\\
      \hline
      x & 2\\
      \hline
    \end{tabular}
    }
  \end{subtable}%
  \caption{Solutions to instance in Figure \ref{fig:counter2}.}
  \label{tab:solutions2}
\end{table}

Figure \ref{fig:counter2} shows a $(2,1)$-stable instance for which
$\alpha$-expansion does not always find the optimal solution. Table
\ref{tab:solutions2} shows the optimal integer solution and the
solution returned when $\alpha$-expansion is run with all labels
initially set to 2. The LP relaxation is tight because this instance
is a tree \parencite{wainwrightjordan}. Suppose $\alpha$-expansion
starts with every node assigned the label 2. In the first iteration,
it finds the optimal expansion of label 1, and assigns label 1 to $u$,
$v$, and $w$, but leaves $x$ labeled 2. After this move, no expansions
can decrease the objective, so the algorithm terminates. The instance
is $(2,1)$-stable because in the adversarial perturbation (where edges
$(u,v)$ and $(w,x)$ have weights 2 and 1.5, and $(v,w)$ has weight 3),
the optimal solution is still to label $u$, $v$ with 2 and $w$, $x$
with 3.

\section{CONCLUSION}
\label{sec:discuss}
We gave conditions under which two popular algorithms for MAP
inference in Ferromagnetic Potts Models are exact.
The results in Section \ref{sec:strongstab} provide a possible avenue for
explaining why LP relaxations are often tight in practice.
For weakly stable instances, the results in Section
\ref{sec:weakstab} provide a possible explanation for the observed
phenomena that the solutions output by $\alpha$-expansion are often
visually indistinguishable from the optimal solution and that the
output does not heavily depend on the choice of initial labeling,
since we prove that the algorithm always recovers the optimal solution on
the stable set $S$ regardless of the initialization. 
Note that on $(2,2)$-stable instances, $\alpha$-expansion
is exact \emph{and} the LP relaxation is tight.

While the $\alpha$-expansion algorithm is a local search algorithm
that essentially does hill-climbing with a particular set of moves,
our results show that it reaches high-quality solutions on stable
instances. This implies the energy landscape of stable instances has
particular properties that make MAP inference tractable, and gives
many directions for future work on understanding the relationship
between stability and optimization.

\subsection*{Acknowledgments}
The authors would like to thank Fredrik D. Johansson for many helpful
discussions and for his feedback on drafts of this paper. This work
was supported by NSF AitF awards CCF-1637585 and CCF-1723344. AV is
also supported by NSF Grant No.~CCF-1652491.

\subsection*{References}
\printbibliography[heading=none]
\newpage
\newpage
\clearpage
\appendix
\section{Supplementary Material}
\subsection{Relaxation on Local Polytope}
The relaxation of \eqref{eq:MRFEM} over the \emph{local polytope} is given by:
\begin{equation*}
  \begin{aligned}
    & \min_{\mu}
    && \sum_{u\in V}\sum_{i \in L}\mu_u(i)c(u,i) + \sum_{e = (u,v)}\sum_{i, j}\mu_{e}(ij)\theta_{(u,v)}(i, j)\\
    & \text{s.t.}
    && \sum_i\mu_u(i) = 1,\ \ \ \ \ \ \ \ \ \ \ \ \  \forall i \in L.\\
    &&& \sum_j \mu_{e}(ij) = \mu_{u}(i), \ \ \ \ \ \ \ \forall e = (u,v) \in E, i \in L.\\
    &&& \sum_i \mu_{e}(ij) = \mu_{v}(j), \ \ \ \ \ \ \ \forall e = (u,v) \in E, j \in L.\\
    &&& \mu_u(i) \ge 0,\ \ \ \ \ \ \ \ \ \ \ \ \ \ \ \ \ \ \ \forall u \in V,\ i \in L.\\
    &&& \mu_e(ij) \ge 0,\ \ \ \ \ \ \ \ \ \ \ \ \ \ \ \ \ \ \forall e \in E,\ i,j \in L.
  \end{aligned}
\end{equation*}
For a Ferromagnetic Potts Model, the objective becomes:
\begin{equation*}
  \min_{\mu} \sum_{u\in V}\sum_{i \in L}\mu_u(i)c(u,i) + \smashoperator{\sum_{e = (u,v)}}w(u,v)\sum_{i,j}\mu_{e}(ij)\mathds{1}(i \ne j)
\end{equation*}
Fix the values $\mu_u(i)$. We want to minimize $$\sum_{e =
  (u,v)}w(u,v)\sum_{i,j}\mu_{e}(ij)\mathds{1}(i \ne j)$$ subject to
the constraints \begin{align*} &\sum_j \mu_{e}(ij) = \mu_{u}(i),
  \ \ \ \ \ \ \ \forall e = (u,v) \in E, i \in L.\\ &\sum_i
  \mu_{e}(ij) = \mu_{v}(j), \ \ \ \ \ \ \ \forall e = (u,v) \in E, j
  \in L.\\ &\mu_e(ij) \ge
  0,\ \ \ \ \ \ \ \ \ \ \ \ \ \ \ \ \ \ \forall e \in E,\ i,j \in L.
\end{align*}
Because $w(u,v) \ge 0$ and $\mu_e(ij) \ge 0$, we want to put as much
mass on $\mu_e(ii)$ as possible without violating a constraint, since
those terms do not appear in the objective. To that end, we set
$\mu_e(ii) = \min(\mu_u(i), \mu_v(i))$. Then using the first
constraint, the objective becomes:
\begin{align*}
&\sum_{e =(u,v)}w(u,v)\sum_{i}\mu_u(i) - \min(\mu_u(i), \mu_v(i))\\
&= \sum_{e =(u,v)}w(u,v)\left(1 - \frac{1}{2}\sum_i\mu_u(i) + \mu_v(i)\right.\\
&+ \left.\sum_i|\mu_u(i)- \mu_v(i)|\right)\\
&= \sum_{e =(u,v)}w(u,v)\sum_i|\mu_u(i) - \mu_v(i)|\\
&= \sum_{e =(u,v)}w(u,v)\frac{|\mu_u - \mu_v|}{2},
\end{align*}
where we use multiple times that $\sum_i \mu_u(i) = 1$. The LP objective is thus:
\begin{equation*}
  \min_{\mu} \sum_{u\in V}\sum_{i \in L}\mu_u(i)c(u,i) + \sum_{e = (u,v)}w(u,v)\frac{|\mu_u - \mu_v|}{2}
\end{equation*}
Identifying $\mu_u$ with $\bar{u}$ and $\mu_v$ with $\bar{v}$, we obtain the LP \eqref{eq:umllp}.

\subsection{Proof of Lemma \ref{lem:roundingguar}}
\begin{proof} This argument is similar to the one in \textcite{TwoStab}.
  First, we verify the last two conditions in Lemma
  \ref{lem:roundingguar}.  Let $\alpha = \frac{2}{k\theta} =
  \frac{5}{3}$ and $\beta = k\theta = \frac{6}{5}$. The algorithm
  clearly returns a feasible solution (i.e. a valid
  labeling). Consider any two vertices $u$ and $v$, and let $\Delta =
  d(u,v)$. There are two cases: $j(u) = j(v)$ and $j(u) \ne j(v)$. In
  the first case, let $j = j(u) = j(v)$. We have $P(u) \ne P(v)$
  exactly when $r \in (\min(\bar{u}_i, \bar{v}_i),
  \max(\bar{u}_i,\bar{v}_i))]$ and $i \ne j$. $r$ is uniformly
    distributed in $(0,\theta)$, so the probability of this occurring
    is $$\mathbb{P}[P(u) \ne P(v)] = \frac{1}{k}\sum_{i:i\ne
      j}\frac{|\bar{u}_i - \bar{v}_i|}{\theta} \le
    \frac{2}{k\theta}d(u,v) = \alpha\Delta.$$ Note that we used $u_i
    \le \varepsilon < \theta$ for $i \ne j$ and for all $u$. Now
    consider the case where $j(u) \ne j(v)$. Here $d(u,v) \ge
    d(e_{j(u)}, e_{j(v)}) - d(u, e_{j(u)}) - d(v, e_{j(v)})$ by the
    triangle inequality ($e_i$ is the $i$th standard basis vector in
    $\mathbb{R}^k)$. So $d(u,v) \ge 1 - 2\varepsilon \ge 1 - 2/30$ for
    $k \ge 3$. So $d(u,v) \ge 14/15$, and $\alpha = 5/3$ so
    $\alpha\Delta > 1$ and the bound trivially applies.

    Next we verify the ``co-appoximation'' condition. First consider
    the case where $j(u) = j(v) = j$. Then $d(u,v) \le d(u, e_{j}) +
    d(e_{j},v) \le 2\varepsilon \le 1/15$. As we showed,
    $\mathbb{P}[P(u) \ne P(v)] \le \alpha\Delta$. So $\mathbb{P}[P(u)
      = P(v)] \ge 1 - \alpha\Delta \ge \beta^{-1}(1 - \Delta)$, where
    the last inequality is because $\frac{1-\beta^{-1}}{\alpha -
      \beta^{-1}} = \frac{1/6}{5/3 - 5/6} = \frac{1}{5} \ge
    \Delta$. Now assume $j(u) \ne j(v)$. Note that if $\bar{u}_i \ge
    r$ and $\bar{v}_i \ge r$, $u$ and $v$ are both added to
    $P_i$. So \begin{align*}\mathbb{P}[P(u) = P(v)] &\ge \mathbb{P}[u_i \ge r, v_i
        \ge r]\\
      &= \frac{1}{k}\sum_{i=1}^k\frac{\min(\bar{u}_i,
      \bar{v}_i)}{\theta}.\end{align*} Here we used that for all $i$,
    $\min(\bar{u}_i, \bar{v}_i) \le \varepsilon < \theta$ since $j(u)
    \ne j(v)$. Then \begin{equation*}
      \begin{split}\mathbb{P}[P(u) = P(v)] \ge
    \frac{1}{k}\sum_{i=1}^k \frac{\bar{u}_i + \bar{v}_i - |\bar{u}_i -
      \bar{v}_i|}{2\theta}\\
    = \frac{1}{k\theta}(1 - d(u,v))
    = \beta^{-1}(1 - d(u,v)).\end{split}\end{equation*} The approximation conditions hold.

    Finally, we check the first two conditions of
    Lemma \ref{lem:roundingguar}. First consider $\mathbb{P}[P(u) = i, i \ne
      j(u)]$. This can only occur when $i$ is selected and $u$ is
    assigned to $P_i$. So $$\mathbb{P}[P(u) = i, i \ne j(u)] =
    \frac{1}{k}\mathbb{P}[\bar{u}_i \ge r] =
    \frac{1}{k}\frac{\bar{u}_i}{\theta} = \frac{5}{6}\bar{u}_i.$$ Now
    we compute $\mathbb{P}[P(u) \ne j(u)]$. A vertex $u$ clearly can
    only be assigned a label $i \ne j(u)$ if such an $i$ is selected
    and $u$ is assigned to it; namely, \begin{align*}\mathbb{P}[P(u)
        \ne j(u)] = \frac{1}{k}\sum_{i:i\ne
        j(u)}\frac{\bar{u}_i}{\theta} &= \frac{1}{k\theta}(1 -
      \bar{u}_{j(u)})\\ &= \frac{5}{6}(1 -
      \bar{u}_{j(u)}).\end{align*} This concludes the proof.
\end{proof}
\subsection{Full Proof of Theorem \ref{thm:strongthm}}
Here we reproduce the proof of Theorem \ref{thm:strongthm} in more detail.
\begin{theorem*}
  On a (2,1)-stable instance of \textsc{Uniform Metric Labeling} with
  optimal integer solution $g$, the LP relaxation \eqref{eq:umllp} is
  tight.
\end{theorem*}
\begin{proof}
Assume for a contradiction that the optimal LP solution
$\{\bar{u}^{LP}\}$ of \eqref{eq:umllp} is fractional. To construct a
stability-violating labeling, we will run Algorithm \ref{alg:rounding}
on a fractional labeling $\{\bar{u}\}$ constructed from
$\{\bar{u}^{LP}\}$ and the optimal integer solution $g$. We then use
Lemma \ref{lem:roundingguar} to show that in expectation, the output
of $\mathcal{R}(\{\bar{u}\})$ must be better than the optimal integer
solution in a particular $(2,1)$-perturbation, which contradicts
$(2,1)$-stability.

Let $\{\bar{u}^g\}$ be the solution to \eqref{eq:umllp} corresponding
to $g$, and define the following $\varepsilon$-close solution
$\{\bar{u}\}$: for every $u$ and every $i$, set $\bar{u}_i =
(1-\varepsilon)\bar{u}^g_i + \varepsilon\bar{u}^{LP}_i$. Note that
$\{\bar{u}\}$ is fractional and $j(u) = g(u)$ for all $u$.

Recall that $E_g$ is the set of edges cut by the optimal solution
$g$. Define the following $(2,1)$-perturbation $w'$ of the weights
$w$:
\begin{equation*}
  w'(u,v) = \begin{cases}
    w(u,v) &(u,v) \in E_g\\    
    \frac{1}{2}w(u,v) &(u,v) \in E\setminus E_g.
  \end{cases}
\end{equation*}
We refer to the objective with modified weights $w'$ as $Q'$ (that is,
$Q'$ is the objective in the instance with weights $w'$ and costs
$c$).

Now let $h = \mathcal{R}(\{\bar{u}\})$. To compare $g$ and $h$, we
will compute $\mathbb{E}[Q'(g) - Q'(h)]$, where the expectation is
over the randomness of the rounding algorithm. By definition,
\begin{align*}
  \mathbb{E}\left[Q'(g) - Q'(h)\right] &= \mathbb{E}[Q'(g) - Q'(h) | h =
    g]\Pr(h = g) \\
  &+ \mathbb{E}[Q'(g) - Q'(h) | h \ne g]\Pr(h \ne g).
\end{align*}
The first term of the sum above is clearly zero. Further, as
$\{\bar{u}\}$ is fractional, the guarantees in Lemma
\ref{lem:roundingguar} imply that $\Pr(h \ne g) > 0$. By
$(2,1)$-stability of the instance, any labeling $h \ne g$ must satisfy
$Q'(h) > Q'(g)$. So stability and fractionality of the LP imply
$\mathbb{E}[Q'(g) - Q'(h)] < 0$.

If we compute $\mathbb{E}[Q'(g) - Q'(h)]$ and simplify using Lemma
\ref{lem:roundingguar} and the definition of $w'$ (see the appendix
for a full derivation), we obtain:
\begin{align*}
  Q'(g) - Q'(h) &= \sum_{u \in V_{\Delta}}c(u,g(u)) + \smashoperator{\sum_{(u,v) \in E_g \setminus E_h}}w'(u,v)\\
    &- \sum_{u \in V_{\Delta}}c(u,h(u)) - \smashoperator{\sum_{(u,v) \in E_h \setminus E_g}}w'(u,v).
\end{align*}
Taking the expectation, we obtain:
\begin{align*}
  \mathbb{E}[Q'(g) - Q'(h)] &= \sum_{u \in V}c(u,g(u))\Pr(h(u) \ne g(u))\\
  &+ \smashoperator{\sum_{(u,v) \in E_g}}w'(u,v)\Pr((u,v) \mbox{ not cut})\\
  &- \sum_{u \in V}\sum_{i \ne g(u)}c(u,i)\Pr(h(u) = i)\\
  &- \smashoperator{\sum_{(u,v) \in E \setminus E_g}}w'(u,v)\Pr((u,v) \mbox{ cut}).
\end{align*}
Applying Lemma \ref{lem:roundingguar} with $j(u) = g(u)$,
\begin{align*}
  \mathbb{E}[Q'(g) - Q'(h)] &\ge \frac{5}{6}\left(\sum_{u \in V}c(u,g(u))(1-\bar{u}_{g(u)})\right.\\
  &+ \left.\smashoperator{\sum_{(u,v) \in E_g}}w'(u,v)(1-d(u,v))\right.\\
  &- \left.\sum_{u \in V}\sum_{i \ne g(u)}c(u,i)\bar{u}_i\right.\\
  &- \left.\smashoperator{\sum_{(u,v) \in E \setminus E_g}}2w'(u,v)d(u,v)\right)
\end{align*}
Using the definition of $w'$,
\begin{equation*}
\begin{split}
  \mathbb{E}[Q'(g) - Q'(h)] \ge \frac{5}{6}\left(\sum_{u \in V}c(u,g(u)) + \smashoperator{\sum_{(u,v) \in E_g}}w(u,v)\right.\\
  - \left.\sum_{u \in V}\sum_{i \in L}c(u,i)\bar{u}_i\right. - \left.\smashoperator{\sum_{(u,v) \in E}}w(u,v)d(u,v)\right)
\end{split}
\end{equation*}
The first two terms are simply $Q(g)$, and the last two are the
objective $Q(\{\bar{u}\})$ of the LP solution $\bar{u}$. Since
$\bar{u} = (1 - \varepsilon)\bar{u}^g + \varepsilon\bar{u}^{LP}$ and
$Q(\{\bar{u}^{LP}\}) \le Q(\{\bar{u}^g\})$, the convexity of the LP
objective implies $Q(\{\bar{u}\}) \le Q(\{\bar{u}^g\}) = Q(g)$. So
$\mathbb{E}[Q'(g) - Q'(h)] \ge 0$. But stability of the instance and
fractionality of the LP solution implied $\mathbb{E}[Q'(g) - Q'(h)] <
0$.
\end{proof}

\subsection{Generating Counterexamples}
The following procedure can be used to find $(\beta, \gamma)$-stable
instances.
\begin{enumerate}
 \item Given a fixed number of nodes $n$ and labels $k$, randomly generate a graph $G$ as follows:
   \begin{enumerate}
     \item Connect any two nodes $(u,v)$ with an edge with probability
       \texttt{connectProb}.
     \item When connecting two nodes, choose the edge weight $w(u,v)$ uniformly at random from $\mathbb{Z} \cap [0, $ \texttt{weightMax}$]$.
   \end{enumerate}
 \item For each node $u$, choose an index $i$ uniformly at random from
   $\{1\ldots k\}$. Draw $c(u,i)$ uniformly at random from $\mathbb{Z} \cap [0,
   $ \texttt{costMax}$]$. Set $c(u,j) = 0$ for $j \ne i$.
 \item Find the optimal solution $g$ to the instance $(G, w, c, L)$.
 \item Let $E_g$ be the set of edges cut by $g$, and consider the
   following adversarial perturbation $w'$ of $w$: \begin{equation*}
     w'(u,v) = \begin{cases} \frac{1}{\beta}w(u,v) & (u,v) \in E
       \setminus E_g\\ \gamma w(u,v) & (u,v) \in E_g
   \end{cases}
 \end{equation*} Let $Q'$ be the objective with these modified weights.
   \item Enumerate the $k^n - 1$ possible labelings not equal to
     $g$. If any of them have $Q'(h) \le Q'(g)$, return to step 1. Otherwise, print $V,
     E, w, c$.
\end{enumerate}
Following this procedure, we can also enforce additional properties of
the instance in step 5 before printing it out. For instance, we can
enforce that the LP must be fractional on the instance, or that
$\alpha$-expansion must not find the optimal solution. If these
additional conditions fail to hold, we return to step 1.

The examples in Section \ref{sec:counter} were found with
\texttt{connectProb} $= 0.5$, \texttt{weightMax} $= 4$,
\texttt{costMax} $= 20$, and then modified for simplicity. Steps 3-5
were repeated for each modification to ensure the resulting instances
satisfied the correct stability conditions. In Section
\ref{sec:lpcounter}, $\beta = 1$ and $\gamma = 2$; in Section
\ref{sec:alphacounter}, $\beta = 2$ and $\gamma = 1$.

The following lemma proves that steps 3-5 are sufficient to verify stability.
\begin{lemma}
  \label{lem:advsuff}
Let $w^*$ be an arbitrary $(\beta, \gamma)$-perturbation of the
weights $w$, and let $w'$ be the adversarial perturbation for the
optimal solution $g$. Then for any labeling $h$, $Q^*(h) \le Q^*(g)$
implies $Q'(h) \le Q'(g)$. In other words, if a labeling $h$ violates
stability in any perturbation, it violates stability in the
adversarial perturbation $w'$.
\end{lemma}
\begin{proof}
We show that $Q^*(g) - Q^*(h) \le Q'(g) - Q'(h)$. Let $V_{\Delta} =
\{u \in V\ |\ g(u) \ne h(u)\}$. Recall that $E_g$ and $E_h$ are the
sets of edges cut by $g$ and $h$, respectively. We compute
\begin{align*}
  Q'(g) - Q'(h) &= \sum_{u \in V_{\Delta}}c(u,g(u)) + \smashoperator{\sum_{(u,v) \in E_g \setminus E_h}}w'(u,v)\\
    &- \sum_{u \in V_{\Delta}}c(u,h(u)) - \smashoperator{\sum_{(u,v) \in E_h \setminus E_g}}w'(u,v).
\end{align*}
Using the definition of $w'$,
\begin{align*}
  Q'(g) - Q'(h) &= \sum_{u \in V_{\Delta}}c(u,g(u)) + \smashoperator{\sum_{(u,v) \in E_g \setminus E_h}}\gamma w(u,v)\\
    &- \sum_{u \in V_{\Delta}}c(u,h(u)) - \smashoperator{\sum_{(u,v) \in E_h \setminus E_g}}\frac{w(u,v)}{\beta}.
\end{align*}
Since $w^*$ is a valid $(\beta, \gamma)$-perturbation,
$\frac{1}{\beta}w(u,v) \le w^*(u,v) \le \gamma w(u,v)$. Then since all
the $c$'s and $w$'s are nonnegative, 
\begin{align*}
  Q'(g) - Q'(h) &\ge \sum_{u \in V_{\Delta}}c(u,g(u)) + \smashoperator{\sum_{(u,v) \in E_g \setminus E_h}}w^*(u,v)\\
    &- \sum_{u \in V_{\Delta}}c(u,h(u)) - \smashoperator{\sum_{(u,v) \in E_h \setminus E_g}}w^*(u,v)\\
    &= Q^*(g) - Q^*(h).
\end{align*}
\end{proof}

\end{document}